\documentclass[twoside]{article}

\usepackage[accepted]{aistats2025}
\usepackage[round]{natbib}
\renewcommand{\bibname}{References}
\renewcommand{\bibsection}{\subsubsection*{\bibname}}

\usepackage{url}
\usepackage{xcolor}
\usepackage{graphicx}
\usepackage{amsmath, amssymb, amsthm}
\usepackage{booktabs}
\usepackage{caption}
\usepackage{float}
\usepackage{hyperref}
\usepackage{algorithmic}
\usepackage{float}
\usepackage{algorithm}
\usepackage{multirow}

\usepackage{amsmath,amsfonts,bm}

\def\Figref#1{Figure~\ref{#1}}

\def\eqref#1{equation~\ref{#1}}
\def\Eqref#1{Equation~\ref{#1}}

\def\1{\bm{1}}

\def\va{{\bm{a}}}
\def\vb{{\bm{b}}}

\def\ve{{\bm{e}}}

\def\evalpha{{\alpha}}

\def\eva{{a}}
\def\evb{{b}}

\def\eve{{e}}

\def\mA{{\bm{A}}}

\def\mC{{\bm{C}}}

\DeclareMathAlphabet{\mathsfit}{\encodingdefault}{\sfdefault}{m}{sl}
\SetMathAlphabet{\mathsfit}{bold}{\encodingdefault}{\sfdefault}{bx}{n}

\def\emA{{A}}

\def\emC{{C}}

\begin{document}

%

%

\newtheorem{theorem}{Theorem}[section]
\newtheorem{lemma}[theorem]{Lemma}
\newtheorem{proposition}{Proposition}
\newtheorem{corollary}[theorem]{Corollary}
\newtheorem{definition}{Definition}
\newtheorem{remark}[theorem]{Remark}
\newtheorem{example}[theorem]{Example}

\def\clater{\textcolor{blue}{TODO( Cite )}}
\newcommand{\todo}[1]{\textcolor{orange}{TODO: #1}}
\twocolumn[

\aistatstitle{Multimodal Learning with Uncertainty Quantification based on Discounted Belief Fusion}

\aistatsauthor{ Grigor Bezirganyan \And Sana Sellami \And  Laure Berti-\'{E}quille \And S\'{e}bastien Fournier}

\aistatsaddress{Aix-Marseille Univ, LIS \And Aix-Marseille Univ, LIS \And IRD, ESPACE-DEV \And Aix-Marseille Univ, LIS} ]

\begin{abstract}
Multimodal AI models are increasingly used in fields like healthcare, finance, and autonomous driving, where information is drawn from multiple sources or modalities such as images, texts, audios, videos. However, effectively managing uncertainty—arising from noise, insufficient evidence, or conflicts between modalities—is crucial for reliable decision-making. Current uncertainty-aware machine learning methods leveraging, for example, evidence averaging, or evidence accumulation underestimate uncertainties in high-conflict scenarios. Moreover, the state-of-the-art evidence averaging strategy is not order invariant and fails to scale to multiple modalities. To address these challenges, we propose a novel multimodal learning method with order-invariant evidence fusion and introduce a conflict-based discounting mechanism that reallocates uncertain mass when unreliable modalities are detected. We provide both theoretical analysis and experimental validation, demonstrating that unlike the previous work, the proposed approach effectively distinguishes between conflicting and non-conflicting samples based on the provided uncertainty estimates, and outperforms the previous models in uncertainty-based conflict detection.
\end{abstract}

\section{INTRODUCTION}

Recently, the use of Artificial Intelligence (AI) has surged dramatically for both autonomous and assisted decision-making. In fields such as healthcare, finance, and autonomous driving, information is often distributed across multiple modalities or views\footnote{The terms modality and view will be used interchangeably in this paper}. To effectively leverage this diverse information, multimodal AI models have been developed \citep{li2024review,xiao2020multimodal,lee2020multimodal,chergui2024mixmas,krones2025review}. However, deep neural networks are prone to making wrong decisions with very high confidence \citep{pmlr-v70-guo17a, Hein_2019_CVPR}. In safety-critical areas, such as healthcare or autonomous driving, this kind of overconfidence can lead to catastrophic results. Hence, when deploying these models in safety-critical areas, it is essential to understand the true uncertainty of the predictions.

Uncertainty in deep learning can appear from multiple sources. Firstly, data can have an inherent amount of randomness (aleatoric uncertainty) that can affect the quality of the model's decisions. Secondly, there can be not enough evidence in the data to make a confident prediction (epistemic uncertainty) \citep{kiureghian_aleatoryepistemic_2009}. Another source of uncertainty that is unique to multimodal networks is the conflict between modalities. We define conflict as confident disagreement between several opinions. Conflicts arise from having unreliable sources of information \citep{martin2019conflict}. In case of multimodal data, the uncertainties may also come from misaligned modalities or malicious injection of incorrect data in the modalities.  

While some multimodal approaches \citep{marcos2013clustering, hou2020fast} opt to remove conflicting samples from the dataset, in many scenarios, it is more beneficial to make decisions based on existing evidences while providing relatively higher uncertainty estimates. \citet{xu2024reliable} has proposed an approach for multi-view conflicting information fusion by evidence averaging (RCML). The authors demonstrate the advantage of their fusion method by showcasing superior classification accuracy through experimental validation under both conflicting and non-conflicting settings. However, \textbf{the evidence averaging operation is defined only for two views and is not associative} \citep{josang2016subjective}, which limits its application to more than 2 views or modalities, making it non-order-invariant. Another shortcoming of evidence averaging is that \textbf{under high levels of conflict, the uncertainty does not increase}, which means that a decision based on two conflictive sources is as reliable as the one based on two non-conflictive sources. This is counter-intuitive, as for example, a decision based on conflicting opinions, shall be less reliable than the decision based on total agreement.  

To address these issues, we propose a novel uncertainty aware multimodal fusion approach: Discounting Belief Fusion (DBF), which utilizes conflictive discounting of unreliable modalities, is order-invariant and can effectively scale to multiple modalities. Additionally, the proposed approach assigns higher uncertainties to conflictive decisions, which improves the reliability assessment of the final decision. We provide an experimental validation of our proposed approach, and open-source the source code in a public repository\footnote{\href{https://github.com/bezirganyan/DBF_uncertainty}{https://github.com/bezirganyan/DBF\_uncertainty}}.

This paper is structured as follows. Section \ref{sec:rel_work} presents the related work in unimodal and multimodal uncertainty quantification, section \ref{sec:method} presents our proposed approach of discounted conflictive fusion, section \ref{sec:exp} presents our experiments and their results, and we conclude the paper with section \ref{sec:concl}.
\section{RELATED WORK}
\label{sec:rel_work}
\subsection{Uncertainty Quantification in Deep Learning}

\citet{gawlikowski2023survey} classified uncertainty quantification (UQ) techniques in deep learning into (1) Bayesian Methods \citep{gal2016dropout, wilson2020bayesian}, (2) Ensemble Methods \citep{lakshminarayanan2017simple}, (3) Single Network Deterministic methods \citep{malinin2018predictive, sensoy2018evidential}, and (4) Test-time augmentation methods \citep{lyzhov2020greedy}. Bayesian methods model uncertainty by placing distributions over weights, which provides better-calibrated predictions and robustness to overfitting, especially in low-data situations. Nevertheless, they often increase the training and inference complexity, require approximations and can be sensitive to the choice of priors \citep{ovadia2019can}. Deep ensembles provide reliable and well-calibrated uncertainty estimates by combining predictions from multiple models, capturing both epistemic and aleatoric uncertainties effectively. However, training and storing multiple independent models in a deep ensemble can be significantly more computationally expensive in terms of memory, training and inference times \citep{lakshminarayanan2017simple}. Test-time Augmentation naturally introduces variability through multiple input transformations, which can provide a straightforward estimate of prediction uncertainty by measuring the variability in outputs \citep{ayhan2018test}. However they significantly increase inference time, do not estimate the epistemic uncertainty and are highly dependent on the augmentation strategy. Single Network Deterministic methods are computationally very efficient and require minimal changes to the network architecture. However, they can struggle to capture the full range of epistemic uncertainty \citep{ulmer2023prior}. 

Evidential Deep Learning (EDL) \citep{sensoy2018evidential}, a single-network deterministic method, has recently gained popularity due to its easy adaptation to existing networks and its low training and testing times \citep{gao2024comprehensive}. EDL replaces the categorical distribution of the last layers of neural networks with a Dirichlet distribution and uses subjective logic \citep{josang2016subjective} to optimize the network parameters. 
Thanks to its ease of adaptation to existing architectures, and computational efficiency, it was also adapted to multi-modal models, which will be described in the next subsection. 

\subsection{Uncertainty Quantification in Multimodal and Multi-view Learning}

While there are many different approaches for uni-modal uncertainty quantification, the multimodal UQ, where we can also have misalignment and conflict between modalities, is relatively less studied. One of the pioneering works, Trusted Multi-view Classification (TMC) \citep{han2021trusted} employs Evidential Deep Learning and uses Dempster's combination (or belief constraint fusion in subjective logic) to combine decisions from multiple modalities. However, as shown by \citet{Zadeh1979}, Dempster's combination rule can yield counter-intuitive results under a high degree of conflict between information sources. \citet{liu_trustedmultiview_2022} address this issue by using evidence accumulation (or cumulative belief fusion in subjective logic) instead of Dempster's combination rule. This type of fusion is suitable for fusing independent information sources, where each source contributes new evidence and thus can reduce uncertainty \citep{josang2016subjective}.

In these approaches, incorporating a new modality decreases uncertainty regardless of the conflict between modalities, while a decision based on conflicting view can be less reliable in practice. To alleviate this issue, \citet{xu2024reliable} propose Evidential Conflictive Multiview Learning (ECML), which introduces an average pooling operation (averaging belief fusion in subjective logic). This method reduces the combined uncertainty if the new view has lower uncertainty and increases it if the new view has higher uncertainty. For two views, the combined uncertainty is the harmonic mean of the individual view uncertainties. However, it is possible to have two conflicting views with low uncertainty, resulting in low combined uncertainty when higher uncertainty would be expected. Additionally, the proposed fusion operation is not associative \citep{josang2016subjective}, making the result of fusing more than two modalities highly dependent on the order of fusion. \citet{huang2025deep} proposed a learnable discounting factor for modality reliability, but it uses a fixed value per modality and class, making it less adaptable to discrepancies between training and deployment. In contrast, our method computes reliability per sample, dynamically adjusting to misalignments and noise that may arise in real-world settings. Similar to these works, we also base our approach on evidential deep learning methods due to their computational efficiency and ease of adaptation. Nevertheless, we address shortcomings of the current approaches by presenting an order-invariant and scalable approach that effectively distinguishes between conflictive and non-conflictive samples based on uncertainty estimates. 

\citet{jung2022uncertainty} and \citet{jung2024beyond} propose alternative multimodal UQ approaches based on Gaussian processes and neural processes, respectively. While these methods perform well, the Multimodal Gaussian Process is computationally expensive due to its non-parametric nature. The Multimodal Neural Process is relatively faster; however, its results are highly dependent on the chosen context set, and there are currently no theoretically guaranteed methods to obtain an optimal context set.

\section{METHODOLOGY}
\label{sec:method}

In this section we will first provide a background on subjective logic, motivate and present the generalized belief averaging fusion operator \citep{josang2017multi}, and finally present our method called Discounted Belief Fusion. 

\subsection{Subjective Logic and Uncertainty}

Evidential Deep Learning uses Subjective Logic to model classification uncertainty by learning evidence scores for each class. Subjective Logic \citep{josang2016subjective} is the extension of Dempster–Shafer's theory of belief functions \citep{dempster1968generalization, shafer1976mathematical} to allow a bijective mapping between a \textit{subjective opinion} and the parameters of Dirichlet distribution. Specifically, for a classification problem of $K$ classes, the subjective opinion $\boldsymbol{\omega}$ is defined as a triplet $\boldsymbol{\omega} = (\vb, u, \va)$, where the $\vb = (\evb_1, \dots, \evb_K)^T$ represents the \textit{belief mass} for each class, $u$ represents the \textit{uncertainty mass} and $\va = (\eva_1, \dots, \eva_K)^T$ represents the \textit{base rate} or prior probability distribution over the classes. In scenarios where no prior knowledge about the classes is available (as is typical in most deep learning applications), a uniform base rate can be assumed. All these terms are strictly non-negative and are less than or equal to one. Unlike probability theory, the belief masses do not sum up to one, however the additivity property dictates that $\sum_{i=1}^K \evb_i + u = 1$. The base rates can usually be set to $1/K$ in case of absence of a prior belief on the classes. The \textit{projected probability distribution} for class $k$ can be obtained with $P_k = \evb_k + \eva_ku, \forall k \in [1, \dots, K]$. 

As mentioned,  subjective logic provides a bijective mapping between the multinomial opinions and the parameters of Dirichlet distribution. The Dirichlet distribution with parameters $\boldsymbol{\alpha} = (\evalpha_1, \dots, \evalpha_K)^T$ is defined by
\begin{equation}
    D(\mathbf{p} \mid \boldsymbol{\alpha})=\left\{\begin{array}{cc}
\frac{1}{B(\boldsymbol{\alpha})} \prod_{k=1}^K p_k^{\alpha_k-1}, & \text { for } \mathbf{p} \in \mathcal{S}_K \\
0, & \text { otherwise, }
\end{array}\right.
\end{equation}
where $\mathbf{p}$ is the probability of each class, $B$ is the multivariate beta function and $\mathcal{S}_K$ is the $K$ dimensional unit simplex. The Dirichlet distribution is a conjugate prior for categorical distribution, meaning that a prior belief encoded by a Dirichlet distribution can be continuously updated by using data with categorical likelihood to obtain updated Dirichlet posterior distribution.  The Dirichlet distribution is a distribution over categorical distributions on $K-1$ dimensional simplex, and can be used to understand the uncertainties present in the data and the model. 

The bijective mapping between the multinomial opinion and Dirichlet distribution can be given with: 
\begin{equation}
b_k=\frac{e_k}{S}=\frac{\alpha_k-1}{S}, u=\frac{K}{S},
\end{equation}
where $S=\sum_{k=1}^K\left(e_k+1\right)=\sum_{k=1}^K \alpha_k$ is called a Dirichlet strength and $\ve = (\eve_1, \dots, \eve_K)^T = \boldsymbol{\alpha} - 1$ is the accumulated evidence that shows the support for each class. Having these parameters, the probability of each class can be computed by the expectation of the Dirichlet distribution with $p_k = \alpha_k / S$. In the next subsections we will adapt and use different fusion strategies from subjective logic to fuse the evidences collected from various modalities. 

\subsection{Generalized Belief Averaging}
\label{subsec:gbaf}
In this subsection, we will theoretically show the shortcoming of the ECML approach \citep{xu2024reliable} of using belief averaging fusion and motivate the use of the generalized belief averaging. 

Let the dataset be defined as ${(x_n^v, y_n)}_{n=1}^{N}, \ \ v \in [1, \dots, V]$, where $x_n^v \in X$ represents the $n$-th data sample from the $v$-th modality, $y_n \in [0, 1]^K$ is the corresponding one-hot encoded label vector, $V$ is the number of modalities, $N$ is the total number of samples, and $K$ is the number of classes. Having $V$ neural networks $f^v(x^v)$ for each view, we obtain the view-specific multinomial opinions $\boldsymbol{\omega}^v = f^v(x^v)$, and the goal is to fuse these opinions into one final opinion: $\boldsymbol{\omega}^{\diamond V} = \boldsymbol{\omega}^1 \diamond \dots \diamond \boldsymbol{\omega}^V$.  

The Evidential Conflictive Multi-view Learning (ECML) \citep{xu2024reliable} uses belief averaging fusion \citep{josang2016subjective} for fusing information from multiple views. The averaging operation for two views is defined as: 
\begin{equation}
\begin{gathered}
\boldsymbol{\omega}^{1 \diamond 2}=\boldsymbol{\omega}^1 \diamond \boldsymbol{\omega}^2=\left(\boldsymbol{b}^{1 \diamond 2}, u^{1 \diamond 2}, \boldsymbol{a}^{1 \diamond 2}\right) \\
b_k^{1 \diamond 2}=\frac{b_k^1 u^2+b_k^2 u^1}{u^1+u^2}, \ \  
u^{1 \diamond 2}=\frac{2 u^1 u^2}{u^1+u^2}, \\  \boldsymbol{a}_k^{A \diamond B}=\frac{\boldsymbol{a}_k^A+\boldsymbol{a}_k^B}{2}
\end{gathered}
\end{equation}
In terms of evidences, the fusion can be defined as 
\begin{equation}
\boldsymbol{e}^{1 \diamond 2}=\frac{1}{2}\left(\boldsymbol{e}^1+\boldsymbol{e}^2\right).
\end{equation}
For multiple views, the fusion can be performed with 
$\boldsymbol{\omega}^{\diamond V}=\boldsymbol{\omega}^1 \diamond \boldsymbol{\omega}^2 \diamond \ldots \diamond \boldsymbol{\omega}^V$.
This averaging belief fusion operator, however, is not associative \citep{josang2016subjective}.

\begin{proposition}
\label{prop:invariance}
    When using an averaging belief fusion operator, the evidence associated with previously fused terms is reduced by a factor of two each time a new term is incorporated, relative to the evidence of the newly fused term.
\end{proposition}
\begin{proof}
The evidence after sequentially fusing $V$ terms using the averaging belief fusion operator, denoted as $\ve^{\diamond V}$, can be computed as follows:
\begin{equation}
\begin{aligned}
\ve^{\diamond V} &= \left( \dots \left( \left(\ve^1+\ve^2\right) \frac{1}{2} + \ve^3 \right) \frac{1}{2} + \dots + \ve^V \right) \frac{1}{2} \\
&= \frac{1}{2^{V-1}} \ve^1 + \frac{1}{2^{V-1}} \ve^2 + \frac{1}{2^{V-2}} \ve^3 + \dots + \frac{1}{2} \ve^V.
\end{aligned}
\end{equation}
The recursive fusion process is initiated by fusing the first two terms, $\ve^1$ and $\ve^2$, with equal weights $\frac{1}{2}$. Each subsequent term $\ve^v$ ($v > 2$) is then added and averaged iteratively. This results in the evidence weight of each previously fused term being halved with every new fusion step, leading to the final expression where the evidence of each term $\ve^v, v \neq 1$ is scaled by $\frac{1}{2^{V-v+1}}$, and $\ve^1$ is scaled by $\frac{1}{2^{V-1}}$. \qedhere
\end{proof}

As an example of Proposition~\ref{prop:invariance}, consider a scenario with three modalities. For a specific class \( k \), the first modality provides an evidence value of 3, the second modality an evidence of 5, and the third modality an evidence of 10. When fusing these evidences using BAF, we compute:

\begin{equation*}
    e_k^{1\diamond2\diamond3} = \frac{\left(\frac{3 + 5}{2}\right) + 10}{2} = \frac{3 + 5}{4} + \frac{10}{2} = 7
\end{equation*}

Here, the evidence values from the first two modalities (3 and 5) are divided by 4, while the evidence value from the third modality (10) is divided by 2. Consequently, the contributions from the first two modalities are weighted only half as much as the third modality. This imbalance leads to a lack of commutativity in the fusion process, meaning that: $e_k^{1\diamond2\diamond3} \neq e_k^{3\diamond2\diamond1}$. To illustrate, if we change the fusion order:
\begin{equation*}
e_k^{3\diamond2\diamond1} = \frac{\left(\frac{10 + 5}{2}\right) + 3}{2} = 5.25
\end{equation*}
This discrepancy is counter-intuitive and significantly hinders the scalability of the algorithm.

To make the fusion process scalable and efficient for many sources, we suggest to use the generalized version of the belief averaging fusion for multiple sources \citep{josang2017multi}, with:
\begin{equation}
\begin{gathered}
\vb^{\diamond V}=\frac{\sum_{v =1} ^V \left(\vb^v \prod_{i \neq v} u^{i}\right)}{\sum_{v =1} ^V \left( \prod_{i \neq v} u^{i}\right)}, \\
u^{\diamond V}=\frac{V \prod_{v=1}^V u^v}{\sum_{v=1}^V \left(\prod_{i \neq v} u^{i}\right)}, \ \
\va^{\diamond V} = \frac{\sum_{v=1}^V \va^v}{V}. 
\end{gathered}
\label{eq:generalized_averaging}
\end{equation}
 The generalized belief averaging fusion in terms of evidences will then be:
\begin{equation}
    \ve ^{\diamond V} = \frac{\sum_{v=1}^V \ve^v}{V}. 
\label{eq:evidence_averaging}
\end{equation}
As we can see, unlike the standard pairwise averaging approach, the generalized belief averaging fusion for multiple sources simultaneously incorporates all $V$ sources, ensuring a uniform weighting scheme. 

\begin{figure*}
    \centering
    \includegraphics[width=1\linewidth]{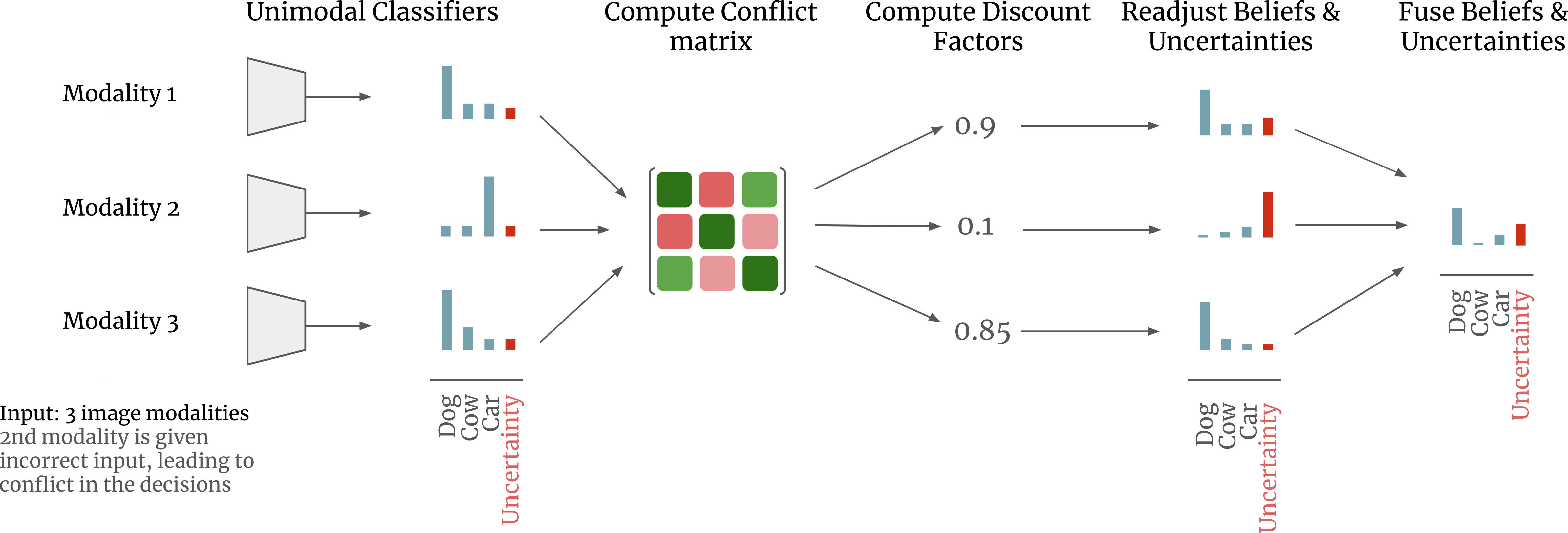}
    \caption{The general pipeline of the Discounted Belief Fusion: First, subjective opinions are being formed for each modality using unimodal classifiers. Then conflict matrix is computed, which is then used to compute discounting factors for each modality. Finally, the beliefs of each modality are being adjusted using the discounting factors and are fused together using generalized belief averaging.}
    \label{fig:dbf-pipeline}
\end{figure*}

\subsection{Discounted Belief Fusion}
In this subsection, we demonstrate that the fusion approach described earlier does not adequately address the increased uncertainty in conflicting samples and propose the Discounted Belief Fusion (DBF) solution to mitigate this issue. The general pipeline of the DBF approach is shown in Figure \ref{fig:dbf-pipeline}.

\begin{proposition}
\label{prop:harmonic_mean}
The uncertainty estimate $u^{\diamond V}$ in \eqref{eq:generalized_averaging} is the harmonic mean of the individual uncertainties $u^1, u^2, \dots, u^V$.
\end{proposition}

\begin{proof}
From \eqref{eq:generalized_averaging}, the uncertainty after fusion is given by:
\begin{equation}
u^{\diamond V} = \frac{V \prod_{v=1}^V u^v}{\sum_{v=1}^V \left(\prod_{i \neq v} u^{i}\right)}.
\end{equation}
Next, observe that the denominator can be rewritten as:
\begin{equation}
\sum_{v=1}^V \left(\prod_{i \neq v} u^{i}\right) = \sum_{v=1}^V \frac{\prod_{i=1}^V u^i}{u^v}.
\end{equation}
Substituting into the expression for $u^{ \diamond V}$, we obtain:
\begin{equation}
u^{\diamond V} = \frac{V \prod_{v=1}^V u^v}{\sum_{v=1}^V \frac{\prod_{i=1}^V u^i}{u^v}} = \frac{V}{\sum_{v=1}^V \frac{1}{u^v}}.
\end{equation}
The expression $\frac{V}{\sum_{v=1}^V \frac{1}{u^v}}$ is the definition of the harmonic mean of the individual uncertainties $u^1, \dots, u^V$. 
\end{proof}

Proposition \ref{prop:harmonic_mean} suggests that the fused uncertainty, as the harmonic mean of individual uncertainties, remains low even in cases of conflict between modalities with low uncertainty. However, in an ideal scenario, conflicting views should result in increased fused uncertainty, reflecting the difficulty and reduced reliability of making a decision in such cases.

\citet{shafer1976mathematical} introduced a discounting operation for combining beliefs in belief theory. The discounting operation is meant to reduce the belief masses of unreliable sources by a constant, and assign the discounted masses to the uncertainty. Since there exists a bijective mapping between belief theory and subjective logic under uniform base rates, we can translate the discounting operation to subjective logic, which will become: 
\begin{equation}
    \vb' = \eta \vb, \ \ \ \ u' = 1 - \eta + \eta u. 
\end{equation}
The $\eta$ represents the degree of the reliability of the evidence. When $\eta = 0$, the updated uncertainty $u'$ becomes 1, indicating complete uncertainty. Conversely, when $\eta = 1$, the uncertainty remains unchanged from its original value, $u$. For intermediate values of $\eta$, the uncertainty $u'$ varies linearly between these two extremes. In our case, we assume that the reliability of the evidences are connected to the degree of conflict between them. In case of 2 views with high conflict we will regard both evidences as unreliable, since based on existing evidence it is not possible to correctly assess which source is truly the reliable one. If we have multiple views,  we want to count how many other views each view conflicts with. If a view is in conflict with many views, it has to be regarded as more unreliable, and has to be discounted more, than views with less conflict. To assess the amount of conflict between the evidences, we will use the Degree of Conflict measure from subjective logic \citep{josang2016subjective}. 

\begin{figure*}
    \centering
    \includegraphics[width=1\linewidth]{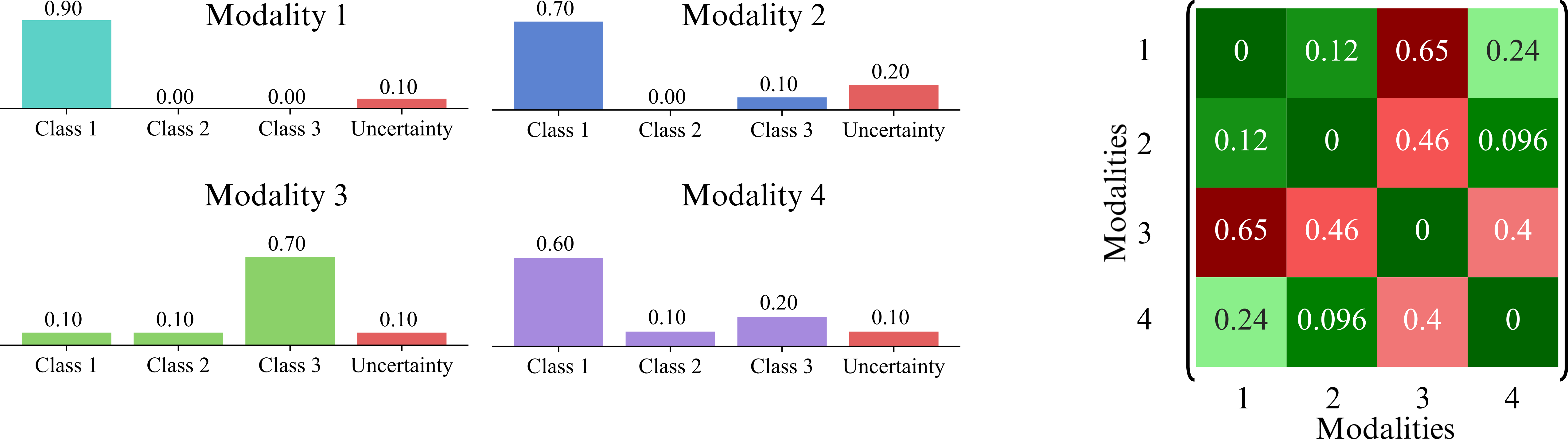}
    \caption{The Conflict Matrix $\mC$ for 4 opinions. }
    \label{fig:conflict-matrix}
\end{figure*}
\begin{figure}
    \centering
    \includegraphics[width=0.8\linewidth]{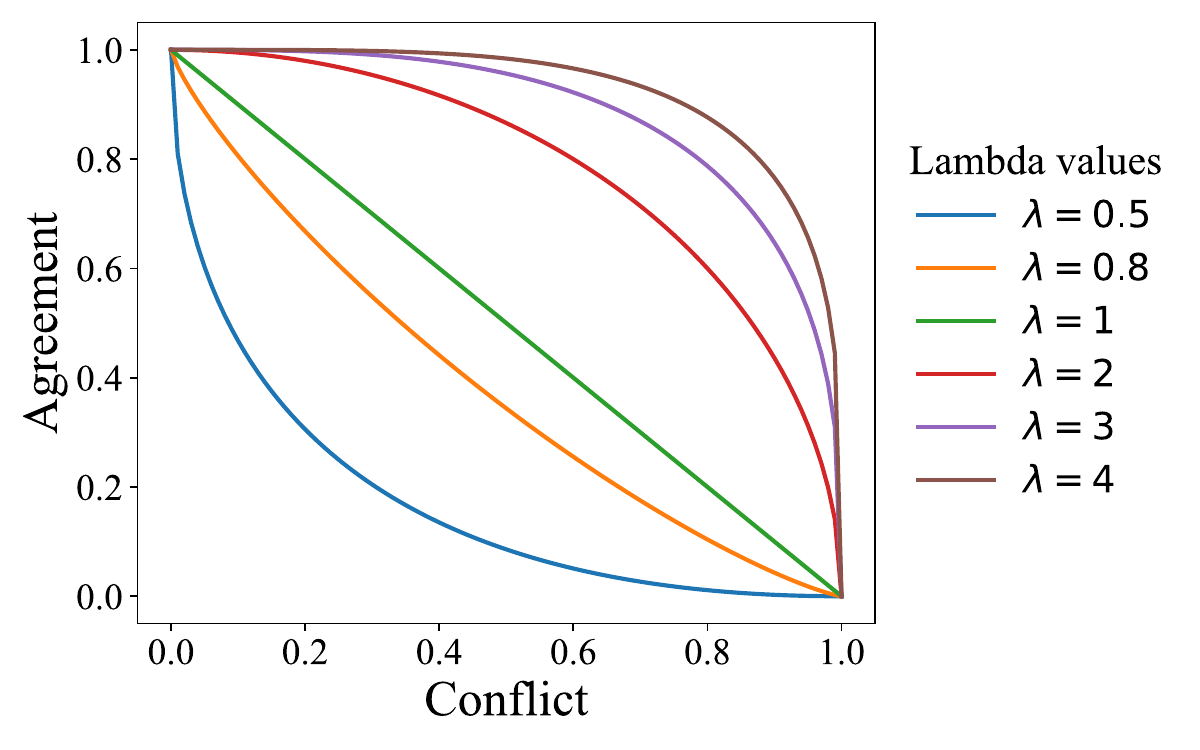}
    \caption{Conflict versus Agreement according to \eqref{eq:updated_discounting_factor}. Higher values of $\lambda$ reduce discounting for lower levels of conflict.} 
    \label{fig:conflict_vs_agreement}
\end{figure}

\begin{definition}[Degree of Conflict]
The degree of conflict $\operatorname{DC}$ between two subjective opinions $\boldsymbol{\omega}^1$ and $\boldsymbol{\omega}^2$ is defined as:
\begin{equation}
\begin{aligned}
\operatorname{PD}\left(\boldsymbol{\omega}^1, \boldsymbol{\omega}^2\right) &= \frac{\sum_{k=1}^K\left|p^1_k - p^2_k\right|}{2}, \\
\operatorname{CC}\left(\boldsymbol{\omega}^1, \boldsymbol{\omega}^2\right) &= \left(1-u^1\right)\left(1-u^2\right), \\
\operatorname{DC}\left(\boldsymbol{\omega}^1, \boldsymbol{\omega}^2\right) &= \operatorname{PD}\left(\boldsymbol{\omega}^1, \boldsymbol{\omega}^2\right) \cdot \operatorname{CC}\left(\boldsymbol{\omega}^1, \boldsymbol{\omega}^2\right),
\end{aligned}
\label{eq:degree_of_conflict}
\end{equation}
where $p_k^v$ denotes the projected probability for class $k$ under opinion $v$, and $u^v$ represents the uncertainty of opinion $v$. The function $\operatorname{PD}\left(\boldsymbol{\omega}^1, \boldsymbol{\omega}^2\right)$ is the projected distance between two opinions, which measures the dissimilarity between the projected probabilities of the opinions and equals zero only when both sets of probabilities are identical. The function $\operatorname{CC}\left(\boldsymbol{\omega}^1, \boldsymbol{\omega}^2\right)$ is the conjunctive certainty between the two opinions, which takes a value of one when both opinions are fully confident (i.e., $u^1 = u^2 = 0$) and approaches zero as the uncertainty of one or both opinions increases.
\end{definition}

Analyzing the Degree of Conflict, we observe that it is significantly high only when two conditions are met: (1) both opinions exhibit low uncertainty, and (2) the projected probabilities of these opinions differ substantially, indicating a high distance between them. \textbf{Consequently, uncertain modalities cannot contribute to high conflict, as conflict arises predominantly between opinions that are both confident and significantly divergent.}

\Eqref{eq:degree_of_conflict} defines the degree of conflict for two opinions. We will use this metric to understand how conflicting an opinion is with the other opinions. For that, we compute the pairwise degree of conflicts between each pair of opinions, and store them in a matrix $\mC$ (see \Figref{fig:conflict-matrix} as an example). Then, we can compute an agreement matrix with $\mA = 1 - \mC$, and we can compute the discounting factor $\eta^v$ as the product of the row elements of the agreement matrix with:
\begin{equation}
\begin{aligned}
    \emC_{ij} = \operatorname{DC}(\boldsymbol{\omega}^i, \boldsymbol{\omega}^j), \ \ \mA = 1 - \mC, \\ 
\eta^v = \prod_{i=1}^{V} \emA_{vi}, \ \ \forall v \in [1, \dots V].
\label{eq:discounting_computation}
\end{aligned}
\end{equation}
Looking back to the example in \Figref{fig:conflict-matrix}, we can see that by using \eqref{eq:discounting_computation} to compute the discounting factor, the opinions that are in conflict with many other views (opinion 3 in the example), will be discounted much more than the other opinions.
 We can compute the updated belief masses and uncertainties with \eqref{eq:discounting_computation}, and fuse the results with generalized belief averaging (\eqref{eq:generalized_averaging}).

However, due to the product in the \eqref{eq:discounting_computation}, having multiple views can cause the discounting factor get too small, and hence it can increase the uncertainty substantially, even if there is very little conflict between the modalities. To overcome this issue, following the work of \citet{martin2019conflict} we modify the agreement matrix calculation \eqref{eq:discounting_computation} to be: 
\begin{equation}
    \mA = (1 - \mC ^ \lambda )^{1/\lambda},
    \label{eq:updated_discounting_factor}
\end{equation}
where $\lambda$ is a hyperparameter controlling how strong the discounting shall be relative to conflict. In \Figref{fig:conflict_vs_agreement} the relation between conflict and agreement under different values of $\lambda$ are shown. 

To illustrate the differences between the fusion functions, we will use the famous example of \citet{Zadeh1979}, adapted to a multimodal setting. Here, two conflicting modalities allocate belief masses as follows: the first source assigns $0.99$ to class 1 and $0.01$ to class 3, while the second assigns $0.99$ to class 2 and $0.01$ to class 3. As we can see in Table \ref{tab:fusion_comparison}, despite class 3 being unlikely for both, Dempster’s rule (Belief Constrained Fusion) combines them into class 3 with full confidence, ignoring the conflict. The Cumulative Belief Fusion and the Averaging Belief Fusion provide better results in this case, by splitting most of the belief mass between the classes 1 and 2, however, the uncertainty is still 0, which means they are still confident in the decision.  In contrast, discounting fusion shifts most belief mass to uncertainty, indicating the unreliability of the evidences acquired, while the parameter $\lambda$ in \eqref{eq:updated_discounting_factor} controls the strength of discounting based on conflict. 

\begin{table}
\caption{Beliefs and Uncertainty values for fusing two modalities' opinions with Belief Constraint Fusion (BCF), Cumulative Belief Fusion (CMF), Belief Averaging Fusion (BAF), and our proposed Discounting Belief Fusion (DBF).}
\centering
\begin{tabular}{@{}lcccc@{}}
\toprule
Fusion Method    & \textbf{$\evb_1$} & \textbf{$\evb_2$} & \textbf{$\evb_3$} & \textbf{$u$} \\ 
\midrule
Modality 1     & 0.99   & 0.00   & 0.01   & 0 \\
Modality 2     & 0.00   & 0.99   & 0.01   & 0 \\
\midrule
BCF  & 0.0000 & 0.0000 & 1.0000 & 0 \\ 
CBF    & 0.4950 & 0.4950 & 0.0100 & 0 \\ 
BAF     & 0.4950 & 0.4950 & 0.0100 & 0 \\ 
DBF $\lambda = 1$  & 0.0050 & 0.0050 & 0.0001 & $0.9900$                \\ 
DBF $\lambda = 3$  & 0.1533 & 0.1533 & 0.0031 & $0.6903$                \\ 

\bottomrule
\end{tabular}
\label{tab:fusion_comparison}
\end{table}

\subsection{Training Multimodal Evidential Neural Networks}
To fuse information from different modalities as described in the previous subsections, we first need to extract subjective opinions from each modality. For that, we construct evidential neural networks for each modality, which output non-negative evidence values to be later combined using a chosen fusion function. Simultaneous availability of the modalities during training is preferred but not mandatory. If simultaneous availability is not possible, one can effectively integrate new modalities without complete retraining, by training only a classifier for the new modality, and if necessary, fine-tuning the other modalities.

Transforming a traditional deep neural network into an evidential one is straightforward. It involves replacing the softmax activation function in the output layer with a non-negative, monotonically increasing activation function. Common choices include ReLU, softplus, and exponential functions. While \citet{xu2024reliable} utilized a softplus activation layer, our experiments revealed that the slow growth of the softplus function resulted in small evidence values, leading to consistently high uncertainties across modalities. Recently, \citet{pandey2023learn} demonstrated the superiority of the exponential activation function over other alternatives for evidential networks. Hence, we employ the exponential activation function to obtain evidence values. To mitigate numerical instability, we cap the activation values at $10^{13}$ with $f(x) = \frac{10^{13}}{1+10^{13}e^{-x}}$.

For training the neural network, we follow \citet{xu2024reliable} to use 3 part loss, since it also aims to minimize the conflict between modalities. The first part, is the adapted cross-entropy loss per modality and for fused result: 
\begin{equation}
    \begin{aligned}
L_{a c e}\left(\boldsymbol{\alpha}_n\right) & =\int\left[\sum_{j=1}^K-y_{n j} \log p_{n j}\right] \frac{\prod_{j=1}^K p_{n j}^{\alpha_{n j}-1}}{B\left(\boldsymbol{\alpha}_n\right)} d \mathbf{p}_n \\
& =\sum_{j=1}^K y_{n j}\left(\psi\left(S_n\right)-\psi\left(\alpha_{n j}\right)\right),
\end{aligned}
\end{equation}
where $\psi(\cdot)$ is the digamma function, and $\boldsymbol{\alpha}_n$ are the Dirichlet parameters of the $n$-th sample. To force the evidences of incorrect labels to be lower, a regularization term is employed to minimize the Kullback-Leiber divergence of modified Dirichlet distribution parameters and the uniform Dirichlet distribution: 
\begin{equation}
\begin{aligned}
L_{K L}\left(\boldsymbol{\alpha}_n\right) & =K L\left[D\left(\boldsymbol{p}_n \mid \tilde{\boldsymbol{\alpha}}_n\right) \| D\left(\boldsymbol{p}_n \mid \mathbf{1}\right)\right] \\
& =\log \left(\frac{\Gamma\left(\sum_{k=1}^K \tilde{\alpha}_{n k}\right)}{\Gamma(K) \prod_{k=1}^K \Gamma\left(\tilde{\alpha}_{n k}\right)}\right) \\
& +\sum_{k=1}^K\left(\tilde{\alpha}_{n k}-1\right)\left[\psi\left(\tilde{\alpha}_{n k}\right)-\psi\left(\sum_{j=1}^K \tilde{\alpha}_{n j}\right)\right],
\end{aligned}
\end{equation} where $\tilde{\alpha}_n = \mathbf{y}_n+\left(\mathbf{1}-\mathbf{y}_n\right) \odot \boldsymbol{\alpha}_n$ are the Dirichlet parameters after removing the non-misleading evidence, and $\Gamma(\cdot)$ is the gamma function. We can gradually increase the effect of the KL Divergence loss with an annealing coefficient, to allow the network to explore the parameter, and avoid early convergence of the misclassified samples to the uniform distribution \citep{sensoy2018evidential}. 
The annealing coefficient $\sigma_t$ can be chosen as $\sigma_t=\min (1.0, t / T) \in[0,1]$, where $t$ is the current epoch and $T$ is the annealing step. 
\begin{equation}
    L_{a c c}\left(\boldsymbol{\alpha}_n\right)=L_{a c e}\left(\boldsymbol{\alpha}_n\right)+\sigma_t L_{K L}\left(\boldsymbol{\alpha}_n\right)
\end{equation}
Finally, a consistency loss is employed to minimize the degree of conflict between the modalities: 
\begin{equation}
    L_{\text {con }}=\frac{1}{V-1} \sum_{p=1}^V\left(\sum_{q \neq p}^V \operatorname{DC}\left(\boldsymbol{\omega}_n^p, \boldsymbol{\omega}_n^q\right)\right)
\end{equation}
The total loss is calculated with: 
\begin{equation}
L=L_{a c c}\left(\boldsymbol{\alpha}_n\right)+\beta \sum_{v=1}^V L_{a c c}\left(\boldsymbol{\alpha}_n^v\right)+\gamma L_{c o n} .
\end{equation}
\section{EXPERIMENTS}
\label{sec:exp}
\subsection{Experimental setup}

Five multimodal / multi-view datasets will be used for comparison: HandWritten~\citep{multiple_features_72}, CUB~\citep{WahCUB_200_2011}, Scene15~\citep{Ali_2018_Scene15}, Caltech101~\citep{li_andreeto_caltech_2022} and PIE~\citep{gross_2008_multi_pie}.
Please check the supplementary material for the descriptions and details about the datasets. 
As baselines we use three evidential multi-view methods with different opinion aggregation strategies. The Trusted Multiview Classification \citep{han2021trusted} uses Dempster's combination rule or Belief Constraint Fusion (BCF), Trusted Multiview Deep Learning with Opinion Aggregation \citep{liu_trustedmultiview_2022} uses Cumulative Belief Fusion (CBF), Evidential Conflictive Multiview Learning \citep{xu2024reliable}  uses the Belief Averaging Fusion (BAF) and the Generalized Belief Averaging Fusion (GBAF, described in subsection \ref{subsec:gbaf}).

We conduct all training and evaluation on a local cluster equipped with a combination of Nvidia GeForce RTX 2080 and RTX 4090 GPUs. Each experiment is repeated 10 times using different random seeds to ensure robustness. Hyperparameters are selected through 10-fold cross-validation prior to the main experiments. Further details on training procedures and hyperparameter configurations can be found in the supplementary material.  

\subsection{Experimental Results}

\begin{table*}
\centering
\caption{AUC scores of conflictive sample detection based on the uncertainty estimates, using Belief Constraint Fusion (BCF), Cumulative Belief Fusion (CBF), Belief Averaging Fusion (BAF), Generalized Belief Averaging Fusion (GBAF), Discounted Belief Fusion (DBF). $\lambda = 1$ is used for all datasets.}
\label{tab:auc_results}
\begin{tabular}{@{}lccccc@{}}
\toprule
 & CUB & CalTech & HandWritten & PIE & Scene15 \\
\midrule
BCF &  0.53 ± 0.04 &  0.72 ± 0.02 &  0.61 ± 0.01 &  0.36 ± 0.02 &  0.54 ± 0.01 \\
CBF &  0.48 ± 0.05 &  0.56 ± 0.02 &  0.49 ± 0.02 &  0.35 ± 0.03 &  0.53 ± 0.02 \\
BAF &  0.50 ± 0.04 &  0.54 ± 0.02 &  0.51 ± 0.01 &  0.35 ± 0.03 &  0.53 ± 0.02 \\
GBAF &  0.48 ± 0.03 &  0.55 ± 0.02 &  0.51 ± 0.02 &  0.37 ± 0.03 &  0.53 ± 0.01 \\
\textbf{DBF (our)} &  \textbf{0.57 ± 0.06} &  \textbf{1.00 ± 0.00} &  \textbf{0.80 ± 0.02} &  \textbf{0.71 ± 0.04} &  0.53 ± 0.01 \\
\bottomrule
\end{tabular}
\end{table*}
\begin{figure}[t]
    \centering
    \includegraphics[width=1\linewidth]{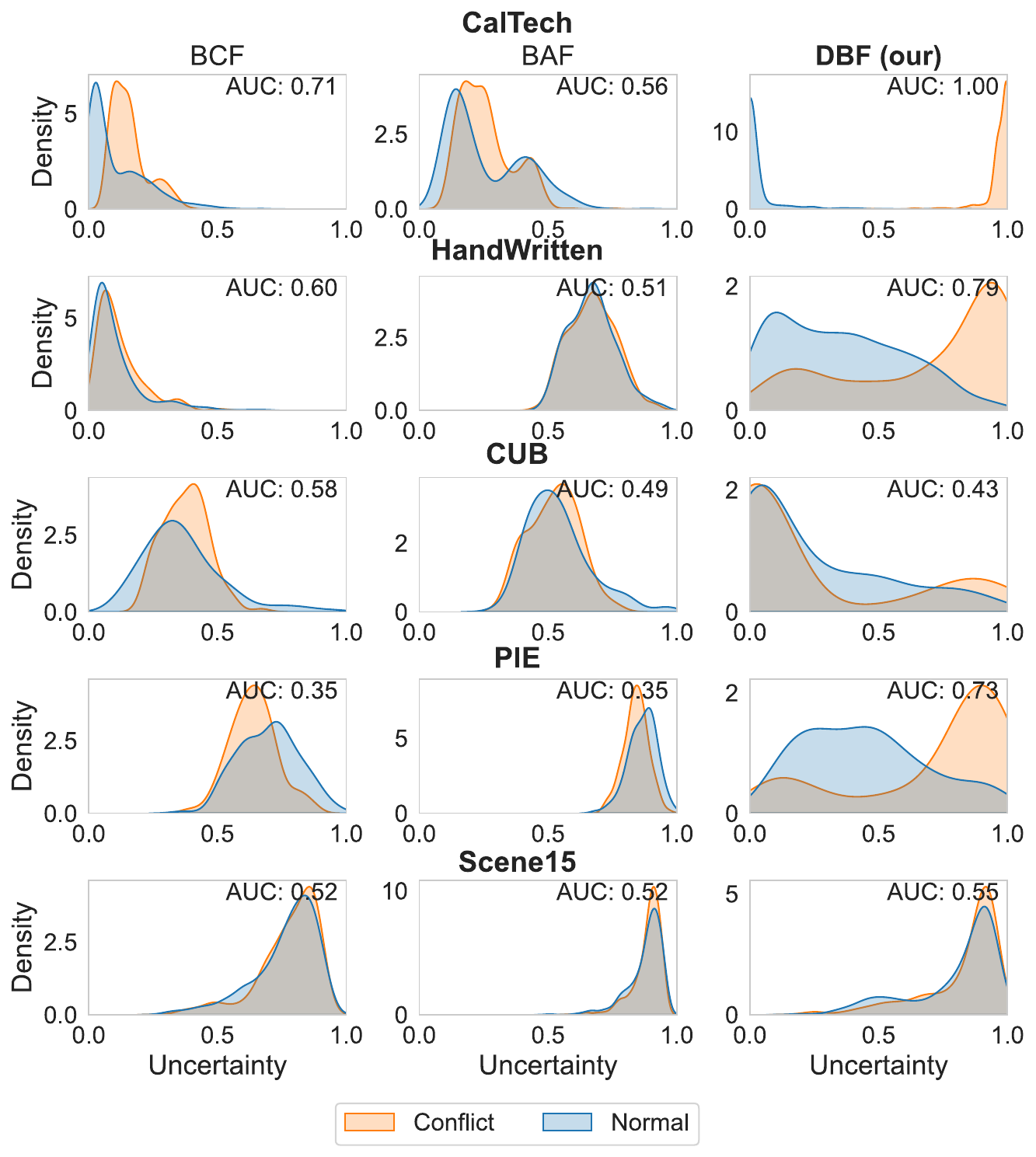}
    \caption{Uncertainty distributions on normal and conflictive test sets using  Belief Constraint Fusion (BCF), Belief Averaging Fusion (BAF), and Discounted Belief Fusion (DBF). $\lambda = 1$ is used for all datasets.}
    \label{fig:uncertainty_plots}
\end{figure}
We evaluate the proposed method by training baseline models on five datasets and testing them on clean and conflicting test sets. Conflicts are introduced by randomly replacing a modality in each sample with one from a different class, simulating misalignment. This setup helps assess whether models produce higher uncertainties on conflicting set compared to clean data.

Figure \ref{fig:uncertainty_plots} and Table \ref{tab:auc_results} show that on most of the benchmark dataset the proposed Discounted Belief Fusion achieves better separation between uncertainties of conflictive and non-conflictive sets compared to the baseline methods. 
However, for the CUB and Scene15 datasets, we observe less effective separation. To understand this, we analyze the uni-modal uncertainties before fusion for each dataset in Figure \ref{fig:modality_uncertainties}. For the CalTech dataset, the uni-modal uncertainties are very low, which leads to a high degree of conflict when misalignment is introduced, resulting in clear separation between conflictive and non-conflictive sets. In contrast, the Scene15 dataset shows high initial uncertainties across modalities, causing low degree of conflict even after misalignment. Thus, although there is misalignment, the conflict remains minimal, and fused uncertainties are high for both sets. For CUB dataset, we observe that one modality has low uncertainty while the other is highly uncertain, hence according to Equation \ref{eq:degree_of_conflict}, the degree of conflict is low since the uncertain modality contributes minimally to the final decision. These results indicate that the proposed approach provides higher uncertainties for conflictive sets when the degree of conflict between modalities is high.

\begin{figure}[t]
    \centering
    \includegraphics[width=1\linewidth]{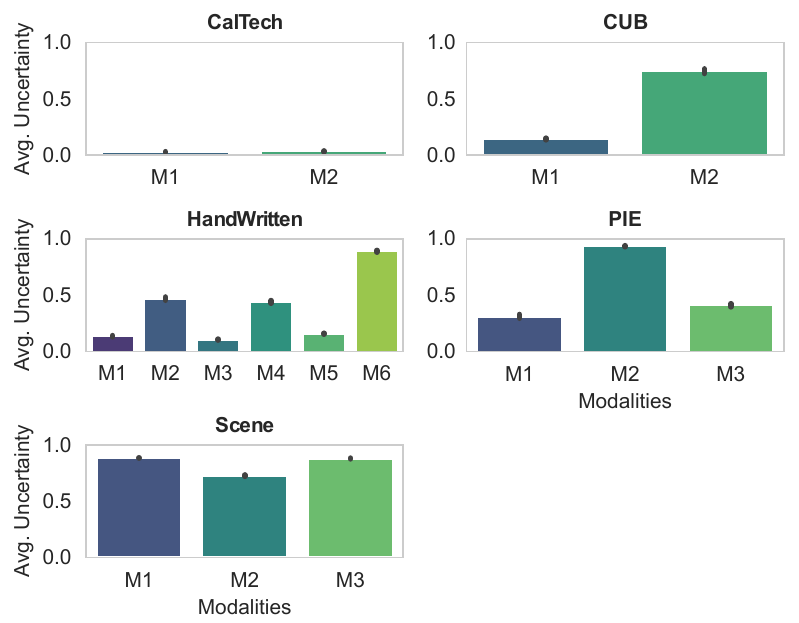}
    \caption{Average uncertainty values across different modalities (denoted as M1, M2, etc.) for CalTech, CUB, HandWritten, Scene, and PIE datasets. Error bars indicate the standard deviation, highlighting the low variability in uncertainty measurements. }
    \label{fig:modality_uncertainties}
\end{figure}
\begin{figure}
    \centering
    \includegraphics[width=1\linewidth]{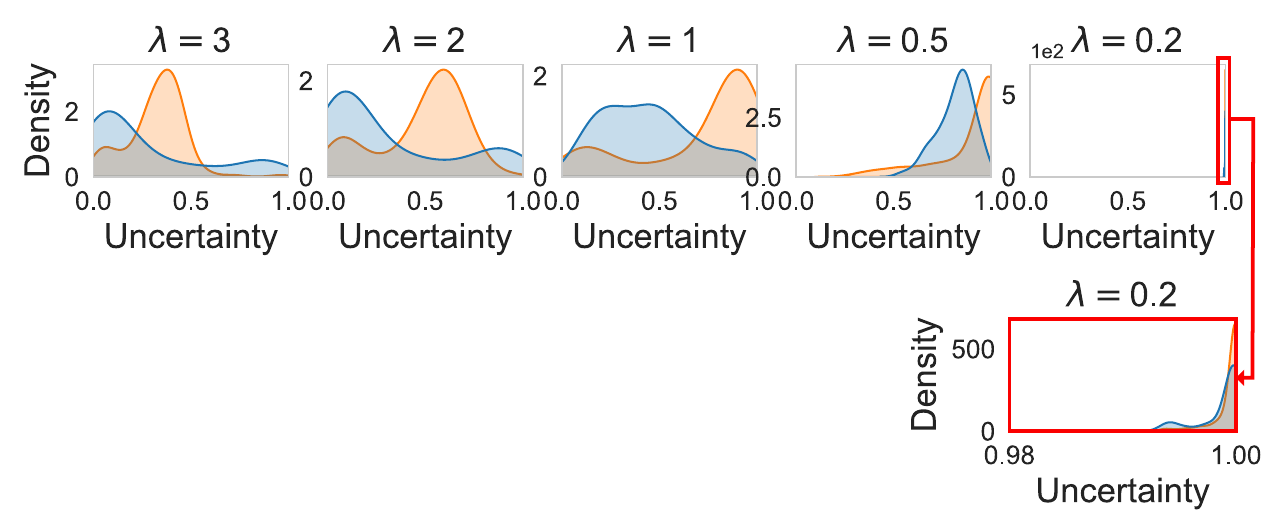}
    \caption{The uncertainty distributions for conflictive and non-conflictive sets with different values of $\lambda$ on PIE dataset.}
    \label{fig:lambda_unc}
\end{figure}
Figure \ref{fig:lambda_unc} demonstrates how the parameter $\lambda$ controls the strictness of discounting. Lower $\lambda$ values lead to more discounting and higher uncertainties. The choice of $\lambda$ can be made using a validation set or based on problem-specific requirements, depending on the desired strictness against conflict.  

In the ablation studies, we tested all models using both softplus (for consistency with prior work) and exponential activation functions (see Appendix \ref{sec:albation} for details). Results show that while the exponential activation improves model performance, it alone does not effectively separate uncertainties between conflictive and non-conflictive samples.

Our experiments showed that the proposed approach does not impact classification accuracy (see Appendix \ref{sec:results}). In most cases, the fusion function had minimal effect on final accuracy, allowing the proposed method to deliver improved uncertainty estimates without compromising the predictive performance of the models.

\begin{table}[h]
    \centering
    \caption{Training times of different algorithms for each dataset. The means and standard deviations of 10 runs are reported in seconds.}
    \begin{tabular}{lccc}
        \toprule
        Dataset & BAF & BCF & DBF (Our) \\
        \midrule
        Scene15     & $757 \pm 22$  & $854 \pm 25$  & $798 \pm 11$  \\
        CalTech     & $413 \pm 2$   & $419 \pm 5$   & $455 \pm 3$   \\
        PIE         & $133 \pm 2$   & $146 \pm 2$   & $140 \pm 2$   \\
        HandWritten & $607 \pm 4$   & $752 \pm 6$   & $602 \pm 12$  \\
        CUB         & $94 \pm 0.5$  & $99 \pm 1$    & $105 \pm 1$   \\
        \bottomrule
    \end{tabular}
    \label{tab:training_times}
\end{table}

As discussed in the Appendix~\ref{sec:algo}, the conflict matrix computation has a quadratic time complexity relative to the number of modalities. However, since the number of modalities is typically not large, the practical time overhead of using DBF remains low. To validate this, we present the training times of various baseline algorithms alongside the DBF method in Table~\ref{tab:training_times}. As observed, DBF is not consistently slower than other algorithms, and even when it is, the difference is not much.

\section{CONCLUSION}
\label{sec:concl}
We introduced Discounted Belief Fusion (DBF), which discounts beliefs from conflicting sources towards uncertainty, achieving better separation between conflictive and non-conflictive samples compared to existing baselines. Extensive experiments on benchmark datasets show that DBF yields higher uncertainty for conflictive samples, especially under high conflict. Our analysis reveals that DBF’s performance adapts to initial modality uncertainty levels, making it suitable for varying alignment quality. Moreover, DBF maintains high classification accuracy, providing robust uncertainty estimates without compromising predictive performance, thereby enhancing interpretability in multimodal learning.

In future work, we plan to further validate our approach using larger benchmark datasets designed for uncertainty-aware multimodal learning, such as the LUMA dataset~\cite{bezirganyan2024luma}. Additionally, we plan to explore modality interactions and dependency structures, which could further enhance our uncertainty quantification strategy.

\subsubsection*{Acknowledgements}
Centre de Calcul Intensif d'Aix-Marseille is acknowledged for granting access to its high-performance computing resources.


\bibliography{bibliography}

\section*{Checklist}

 \begin{enumerate}

 \item For all models and algorithms presented, check if you include:
 \begin{enumerate}
   \item A clear description of the mathematical setting, assumptions, algorithm, and/or model. Yes
   \item An analysis of the properties and complexity (time, space, sample size) of any algorithm. Yes. Please refer to the Appendix \ref{sec:algo}. 
   \item (Optional) Anonymized source code, with specification of all dependencies, including external libraries. Yes
 \end{enumerate}

 \item For any theoretical claim, check if you include:
 \begin{enumerate}
   \item Statements of the full set of assumptions of all theoretical results. Yes
   \item Complete proofs of all theoretical results. Yes
   \item Clear explanations of any assumptions. Yes
 \end{enumerate}

 \item For all figures and tables that present empirical results, check if you include:
 \begin{enumerate}
   \item The code, data, and instructions needed to reproduce the main experimental results (either in the supplemental material or as a URL). Yes, code included in an anonymized repository
   \item All the training details (e.g., data splits, hyperparameters, how they were chosen). Yes. Please refer to Appednix \ref{sec:hyperparams}. 
         \item A clear definition of the specific measure or statistics and error bars (e.g., with respect to the random seed after running experiments multiple times). Yes
         \item A description of the computing infrastructure used. (e.g., type of GPUs, internal cluster, or cloud provider). Yes
 \end{enumerate}

 \item If you are using existing assets (e.g., code, data, models) or curating/releasing new assets, check if you include:
 \begin{enumerate}
   \item Citations of the creator If your work uses existing assets. Yes
   \item The license information of the assets, if applicable. Yes
   \item New assets either in the supplemental material or as a URL, if applicable. Not Applicable
   \item Information about consent from data providers/curators. Not Applicable
   \item Discussion of sensible content if applicable, e.g., personally identifiable information or offensive content. Not Applicable
 \end{enumerate}

 \item If you used crowdsourcing or conducted research with human subjects, check if you include:
 \begin{enumerate}
   \item The full text of instructions given to participants and screenshots. Not Applicable
   \item Descriptions of potential participant risks, with links to Institutional Review Board (IRB) approvals if applicable. Not Applicable
   \item The estimated hourly wage paid to participants and the total amount spent on participant compensation. Not Applicable
 \end{enumerate}

 \end{enumerate}

\onecolumn

\appendix
\aistatstitle{Supplementary Material}

The supplementary material is organized as follows: In Section \ref{sec:algo}, we present the algorithm for the proposed Discounted Belief Fusion (DBF) and its time complexity analysis. Section \ref{sec:hyperparams} provides details on the training process and the hyperparameters used. In Section \ref{sec:data}, we offer additional information about the benchmark datasets employed in the paper. Section \ref{sec:results} includes supplementary results that could not be included in the main paper, and Section \ref{sec:albation} presents ablation studies.

\section{DISCOUNTED BELIEF FUSION ALGORITHM}
\label{sec:algo}

\begin{algorithm}[H]
\caption{Discounted Belief Fusion (DBF) Algorithm}
\label{alg:dbf}
\begin{algorithmic}[1]
    \STATE \textbf{Input:} Set of $V$ opinions $\leftarrow \{\boldsymbol{\omega}^1, \boldsymbol{\omega}^2, \dots, \boldsymbol{\omega}^V\}$ with $K$ classes, hyperparameter $\lambda$
    \FOR{$i = 1$ to $V$}
        \FOR{$j = 1$ to $V$}
            \STATE Compute projected distance: 
            $
            \operatorname{PD}(\boldsymbol{\omega}^i, \boldsymbol{\omega}^j) \gets \frac{1}{2} \sum_{k=1}^K |p^i_k - p^j_k|
            $
            \STATE Compute conjunctive certainty: 
            $
            \operatorname{CC}(\boldsymbol{\omega}^i, \boldsymbol{\omega}^j) \gets (1 - u^i)(1 - u^j)
            $
            \STATE Compute degree of conflict: 
            $
            \operatorname{DC}(\boldsymbol{\omega}^i, \boldsymbol{\omega}^j) \gets \operatorname{PD}(\boldsymbol{\omega}^i, \boldsymbol{\omega}^j) \times \operatorname{CC}(\boldsymbol{\omega}^i, \boldsymbol{\omega}^j)
            $
            \STATE Compute Conflict Matrix: $C_{ij} \gets \operatorname{DC}(\boldsymbol{\omega}^i, \boldsymbol{\omega}^j)$
            \STATE Compute Agreement Matrix: $A_{ij} \gets (1 - (C_{ij})^\lambda)^{1/\lambda}$
        \ENDFOR
    \ENDFOR
    \FOR{$v = 1$ to $V$}
        \STATE Compute discounting factor: 
        $
        \eta^v \gets \prod_{i=1}^V A_{vi}
        $
        \STATE Apply discounting: 
        $
        \boldsymbol{b}^{v'} \gets \eta^v \boldsymbol{b}^v,\ u^{v'} \gets 1 - \eta^v + \eta^v u^v
        $
    \ENDFOR
    \STATE Compute fused belief masses:
    $
    \mathbf{b}^{\diamond V} \gets \frac{\sum_{v=1}^V \left( \boldsymbol{b}^v \prod_{i \neq v} u^i \right)}{\sum_{v=1}^V \left( \prod_{i \neq v} u^i \right)}
    $
    \STATE Compute fused uncertainty: 
    $
    u^{\diamond V} \gets \frac{V \prod_{v=1}^V u^v}{\sum_{v=1}^V \left( \prod_{i \neq v} u^i \right)}
    $
    \RETURN Fused opinion $\boldsymbol{\omega}^\diamond$ with $\boldsymbol{b}^\diamond = [b_1, \dots, b_K]$ and uncertainty $u$
\end{algorithmic}
\end{algorithm}

The time complexity of the proposed Discounting Belief Fusion (DBF) described in Algorithm \ref{alg:dbf} is as follows. Line 4 has a complexity of $\mathcal{O}(K)$, while lines 5-8 each have $\mathcal{O}(1)$ complexity, resulting in a complexity of $\mathcal{O}(V^2K)$ for lines 2-10. Lines 12-13 have a complexity of $\mathcal{O}(V + K)$, leading to a complexity of $\mathcal{O}(V^2 + VK)$ for lines 11-14. Lastly, lines 15-16 contribute $\mathcal{O}(V^2K)$. Therefore, the overall time complexity of the algorithm is $\mathcal{O}(V^2K)$.

\nopagebreak
\section{TRAINING DETAILS AND HYPERPARAMETERS}
\label{sec:hyperparams}
In this section, we provide additional details about the training process for both the baseline models and the proposed approach, with the aim of enhancing the reproducibility of the presented results.

We utilize 2-layer fully connected neural networks with ReLU activations for retrieving the evidences for each modality. In the final layer, we use either a softplus activation function (for baseline models, as proposed in their respective papers) or an exponential activation, 
capped at $10^{13}$ for numerical stability. To maintain differentiability, the activation function is modified to $ f(x) = \frac{10^{13}}{1 + 10^{13}e^{-x}} $, which behaves exponentially for small inputs but gradually flattens as \(x\) increases.
\pagebreak
To determine the optimal hyperparameters, we conducted a hyperparameter search using 10-fold cross-validation. To minimize time and resource consumption, this search was only performed on the ECML model \cite{xu2024reliable}. This ensures that any slight performance differences due to hyperparameters would favor the baseline rather than our proposed approach. The selected hyperparameters are presented in Table \ref{tab:hyperparameters}. All experiments have been performed 10 times with different random seeds, and the average and the standard deviation values have been reported in the results. 

\begin{table}
    \caption{Hyperparameters for Different Datasets}
    \centering
    \begin{tabular}{lccccc}
        \toprule
        \textbf{Dataset} & \textbf{Learning Rate} & \textbf{Annealing Step} & \textbf{Gamma} & \textbf{Weight Decay}\\
        \midrule
        CUB & 0.003 & 10 & 1 & 0.0001  \\
        HandWritten & 0.003 & 50 & 0.7 &  0.00001  \\
        PIE & 0.003 & 50 & 0.5 & 0.00001  \\
        CalTech & 0.0003 & 30 & 0.5 & 0.0001  \\
        Scene & 0.01 & 30 & 0.5 & 0.00001  \\
        \bottomrule
    \end{tabular}
    \label{tab:hyperparameters}
\end{table}

\section{DETAILS ABOUT THE DATASETS}
\label{sec:data}

\textbf{HandWritten} \citep{multiple_features_72} is a small dataset comprised of 2,000 instances of handwritten digits. Each digit is represented using 6 different features, and each feature will be used as a modality. The dataset is licensed under \href{https://creativecommons.org/licenses/by/4.0/legalcode}{CC BY 4.0} License. 

\textbf{CUB} \citep{WahCUB_200_2011} is a textual dataset comprised of descriptions of 11,788 bird belonging to 200 different categories. To be consistent with the previous research, we will use only the first 10 categories and extract the HOG and MBH features as two views of the dataset.

\textbf{Scene15} \citep{Ali_2018_Scene15} consists of 4485 images of 15 indoor and outdoor scene categories. GIST, PHOG and LBP features will be used as different views of the dataset. The dataset is licensed under \href{https://creativecommons.org/licenses/by/4.0/legalcode}{CC BY 4.0} License. 

\textbf{Caltech101} \citep{li_andreeto_caltech_2022} consists of 8677 images from 101 different classes. To be consistent with the previous research, we will use only the first 10 classes and use features extracted from DECAF and VGG19 models as different views of the dataset. The dataset is licensed under \href{https://creativecommons.org/licenses/by/4.0/legalcode}{CC BY 4.0} License. 

\textbf{PIE} \citep{gross_2008_multi_pie} contains 680 instances from 68 different classes. The intensity, LBP and Gabor features are used as 3 different views of the dataset. The \href{https://wlb-prod-us06-user-uploads.s3.us-west-1.amazonaws.com/88f59f5065c26163cad1638f052b269f.pdf?X-Amz-Algorithm=AWS4-HMAC-SHA256&X-Amz-Credential=AKIAIAUSRBZZ2WPI57ZQ%2F20241015%2Fus-west-1%2Fs3%2Faws4_request&X-Amz-Date=20241015T122642Z&X-Amz-Expires=900&X-Amz-SignedHeaders=host&X-Amz-Signature=5a0fb558b5fed4b1ff14314ed7c60fc140b0993dc561a7b597364050ec5e2335}{license} of the dataset allows it to be used only for noncommercial research purposes. 
\section{ADDITIONAL RESULTS}
\label{sec:results}
In Tables \ref{tab:acc_results} and \ref{tab:conf_results} please find the classification accuracy scores on normal and conflictive test sets respectively.  As we can see, the proposed belief constraint fusion does not affect the classification accuracy. In fact, the results indicate, that none of the fusion functions affects the classification accuracy in a statistically significant manner. The Discounted Belief Fusion, however, provides better separation for conflicting and non-conflicting uncertainties as seen in the main paper.

\begin{table}[H]
\centering
\caption{Mean Accuracy (\%) and Standard Deviation for Different Aggregations on Normal test set}
\label{tab:acc_results}
\begin{tabular}{lccccc}
\toprule
 &           CUB &       CalTech &   HandWritten &           PIE &         Scene \\
 \midrule
 \textbf{BCF        } &  91.08 ± 3.23 &  99.45 ± 0.20 &  97.97 ± 0.55 &  94.78 ± 1.75 &  72.03 ± 1.22 \\
\textbf{CBF        } &  91.58 ± 3.50 &  99.52 ± 0.17 &  97.88 ± 0.52 &  92.57 ± 2.66 &  72.93 ± 1.69 \\
\textbf{BAF   } &  91.92 ± 2.79 &  99.36 ± 0.23 &  97.53 ± 0.59 &  93.82 ± 1.68 &  71.52 ± 1.76 \\
\textbf{GBAF    } &  90.25 ± 2.27 &  99.05 ± 0.49 &  97.85 ± 0.45 &  95.07 ± 1.09 &  72.26 ± 1.42 \\
\textbf{DBF (our)        } &  89.33 ± 2.60 &  99.34 ± 0.34 &  97.60 ± 0.71 &  93.90 ± 1.68 &  73.14 ± 1.67 \\
\bottomrule
\end{tabular}
\end{table}

\begin{table}[H]
\centering
\caption{Mean Accuracy (\%) and Standard Deviation for Different Aggregations on Conflicting test set}
\label{tab:conf_results}
\begin{tabular}{lccccc}
\toprule
 &           CUB &       CalTech &   HandWritten &           PIE &         Scene \\
\midrule
\textbf{BCF        } &  52.42 ± 3.22 &  61.68 ± 1.47 &  96.97 ± 0.88 &  91.40 ± 2.13 &  57.17 ± 1.18 \\
\textbf{CBF        } &  53.50 ± 4.94 &  61.00 ± 1.32 &  95.03 ± 0.99 &  90.81 ± 2.24 &  55.79 ± 1.04 \\
\textbf{BAF   } &  50.83 ± 4.50 &  60.48 ± 1.93 &  90.17 ± 1.71 &  84.41 ± 2.65 &  56.38 ± 1.84 \\
\textbf{GBAF    } &  49.25 ± 3.32 &  61.32 ± 1.68 &  95.67 ± 0.51 &  91.69 ± 1.68 &  56.15 ± 0.95 \\
\textbf{DBF (our)        } &  52.33 ± 4.91 &   61.51 ± 2.10  &  97.58 ± 0.74 &  90.29 ± 1.70 &  57.76 ± 1.29 \\
\bottomrule
\end{tabular}
\end{table}

In Figure \ref{fig:full_uncertainty_distributions}, the uncertainty distributions for normal and conflicting sets are shown across different fusion methods (plots for CBF and GBAF were omitted from the main paper due to space limitation). As we can see, the proposed DBF method continues to exhibit superior separation between conflicting and non-conflicting sets.

\begin{figure}[h]
    \centering
    \includegraphics[width=1\linewidth]{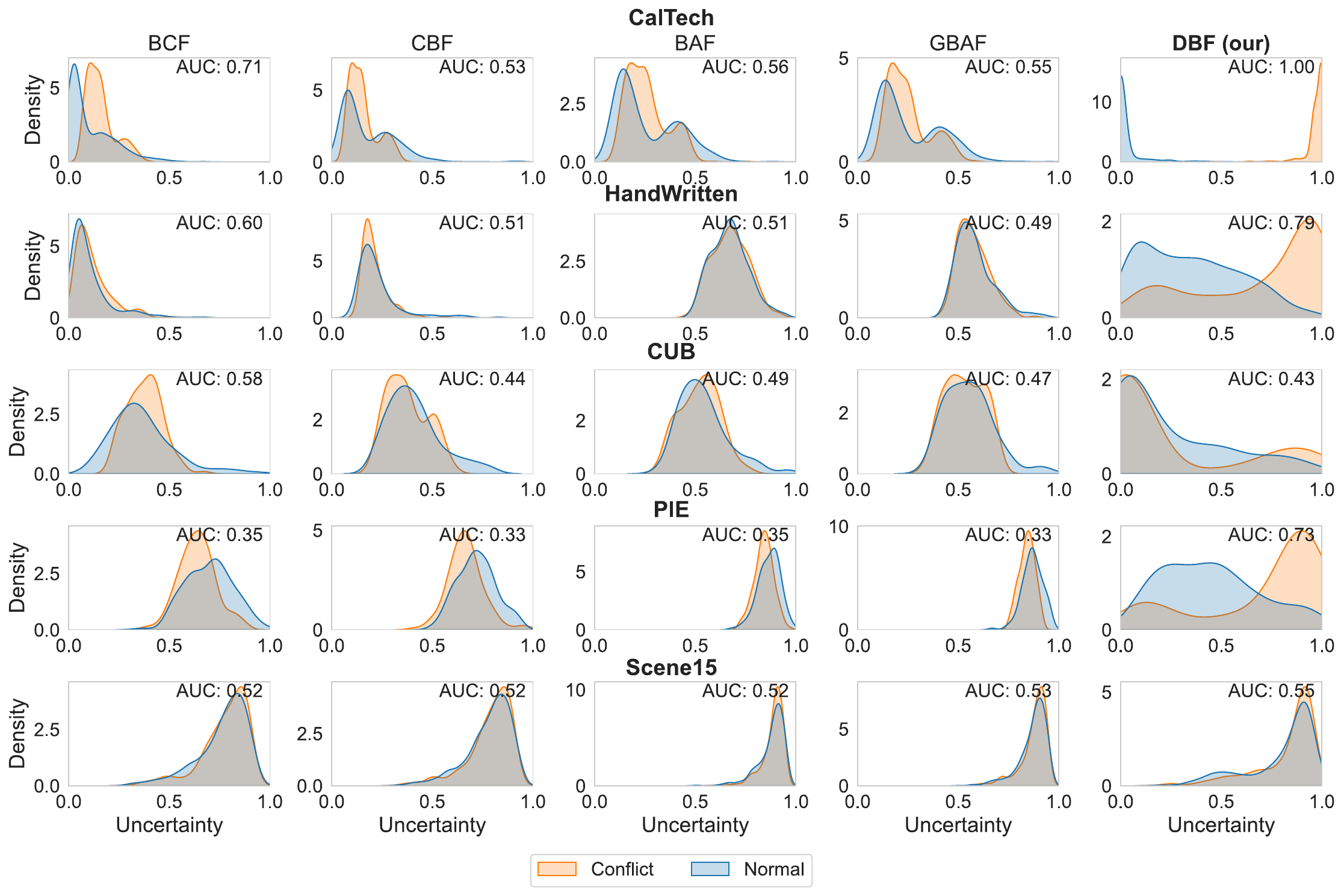}
    \caption{Uncertainty distributions on normal and conflictive test sets using  Belief Constraint Fusion (BCF), Cumulative Belief Fusion (CBF), Generalized Belief Averaging Fusion (GBAF), Belief Averaging Fusion (BAF), and Discounted Belief Fusion (DBF). $\lambda = 1$ is used for all datasets.}
    \label{fig:full_uncertainty_distributions}
\end{figure}

\section{ABLATION STUDIES}
\label{sec:albation}
In this section, we present ablation studies to analyze the effects of both the activation function and the fusion method. As shown in Table \ref{tab:ablation}, the exponential activation function improves performance in Discounted Belief Fusion (DBF) in contrast with softplus activation function. However, it does not yield better results in Belief Averaging Fusion (BAF). This suggests that while the activation function plays a role, it is not solely responsible for better outcomes. The results emphasize the importance of the interaction between both the fusion method and the activation function in the proposed approach.

\begin{table}[t]
\centering
\caption{Ablation study examining the impact of activation functions on the performance of different fusion methods.}
\label{tab:ablation}
\begin{tabular}{lcccc}
\toprule \multirow{2}{*}{\textbf{Dataset}} & \multicolumn{2}{c}{\textbf{Discounted Belief Fusion}} & \multicolumn{2}{c}{\textbf{Belief Averaging Fusion}} \\
\cmidrule(lr){2-3} \cmidrule(lr){4-5} & \textbf{Exponential} & \textbf{Softplus} & \textbf{Exponential} & \textbf{Softplus} \\
\midrule
CUB         &  0.57 ± 0.06 &    0.58 ± 0.03 &  0.44 ± 0.03 &     0.50 ± 0.04 \\
CalTech     &    1.00 ± 0.00 &    0.95 ± 0.01 &  0.41 ± 0.01 &    0.54 ± 0.02 \\
HandWritten &   0.80 ± 0.02 &    0.68 ± 0.02 &  0.47 ± 0.02 &    0.51 ± 0.01 \\
PIE         &  0.71 ± 0.04 &    0.34 ± 0.03 &  0.31 ± 0.02 &    0.35 ± 0.03 \\
Scene       &  0.53 ± 0.01 &    0.53 ± 0.01 &  0.52 ± 0.01 &    0.53 ± 0.02 \\
\bottomrule
\end{tabular}
\end{table}
\vfill

\end{document}